%% file: main.tex
\pdfoutput=1

\documentclass[11pt]{article}

\usepackage[final]{acl}

\usepackage{times}
\usepackage{latexsym}

\usepackage[T1]{fontenc}

\usepackage[utf8]{inputenc}

\usepackage{microtype}

\usepackage{inconsolata}

\input{math_commands.tex}

\usepackage{multirow}
\usepackage{caption}
\usepackage{subcaption}
\usepackage{adjustbox}
\usepackage{color}
\usepackage{wrapfig}
\usepackage{float}
\usepackage{url}
\usepackage{amsthm}
\newtheorem{theorem}{Theorem}[section]

\usepackage{enumitem}

\usepackage{graphicx}
\usepackage{amsmath}
\usepackage{amssymb}
\usepackage{booktabs}

\newcommand{\ignore}[1]{{}}
\newcommand{\mypm}{{\!\!\;\pm\!}}
\newcommand{\reals}{\mathbb{R}}

\newcommand{\xvec}{\mathbf{x}}
\usepackage[normalem]{ulem}

\title{Text2Model: Text-based Model Induction \\for Zero-shot Image Classification}

\author{Ohad Amosy \\ Bar Ilan University, Israel \\ \texttt{amosy3@gmail.com} \\
  \And
  Tomer Volk \\ {\bf Eilam Shapira} \\ {\bf Eyal Ben-David} \\ {\bf Roi Reichart}  \\
  Technion - IIT, Israel \\
  \And
  Gal Chechik \\ Bar Ilan University, Israel \\ NVIDIA Research, Israel \\
  \\
  }

\begin{document}
\maketitle
\begin{abstract}
We address the challenge of building task-agnostic classifiers using only text descriptions, demonstrating a unified approach to image classification, 3D point cloud classification, and action recognition from scenes.
Unlike approaches that learn a fixed representation of the output classes, we \textit{generate at inference time a model} tailored to a query classification task. To generate task-based zero-shot classifiers, we train a hypernetwork that receives class descriptions and outputs a multi-class model. The hypernetwork is designed to be equivariant with respect to the set of descriptions and the classification layer, thus obeying the symmetries of the problem and improving generalization. 
Our approach generates non-linear classifiers, 
handles rich textual descriptions, and may be adapted to produce lightweight models efficient enough for on-device applications.
We evaluate this approach in a series of zero-shot classification tasks, for image, point-cloud, and action recognition, using a range of text descriptions: From single words to rich descriptions. Our results demonstrate strong improvements over previous approaches, showing that zero-shot learning can be applied with little training data.
Furthermore, we conduct an analysis with foundational vision and language models, demonstrating that they struggle to generalize when describing what attributes the class lacks.
\end{abstract}

\section{Introduction}
\label{sec:intro}

\begin{figure}[t]
    \centering
    \includegraphics[width=1.0\linewidth]{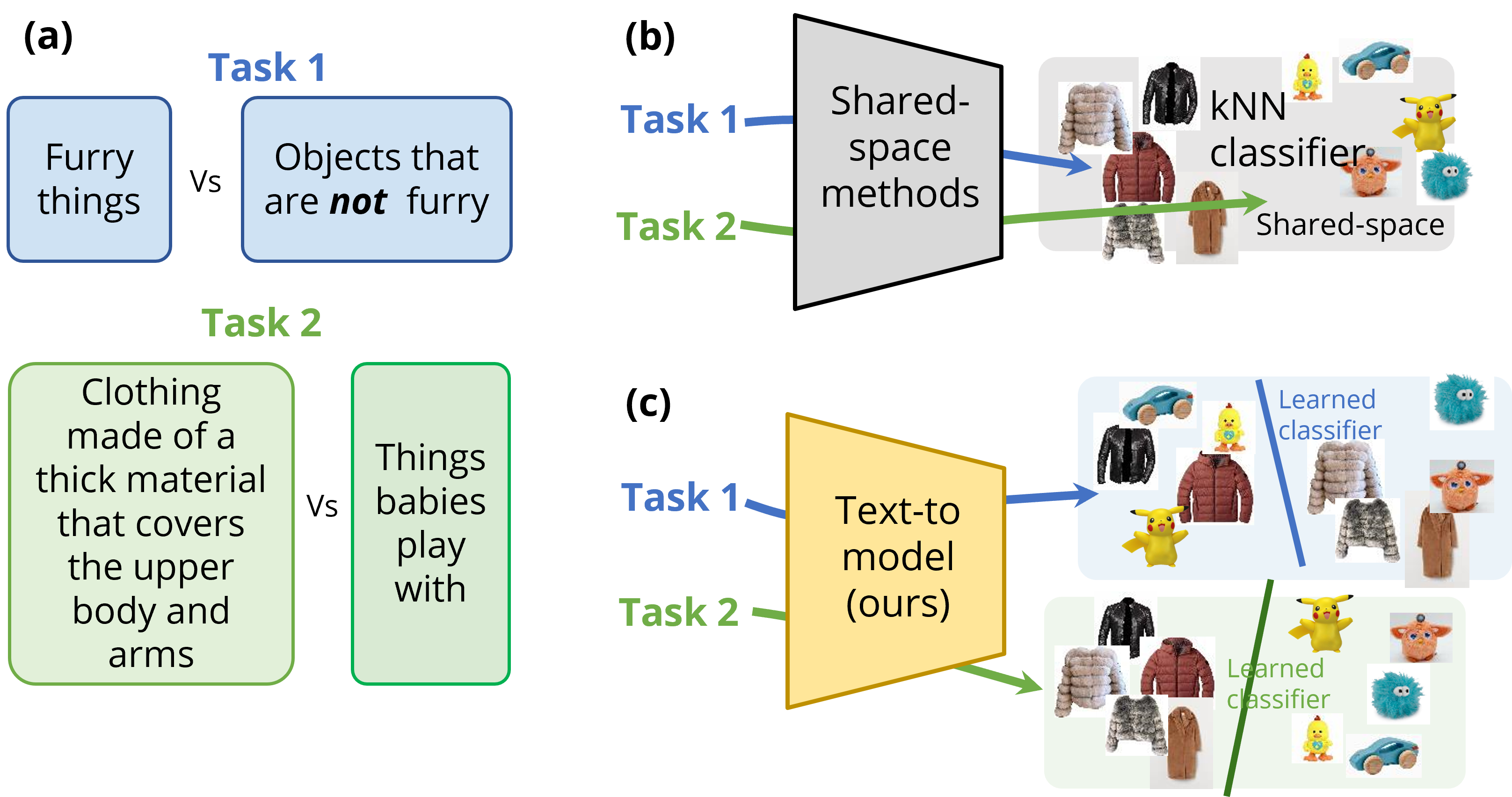}
    \caption{The text-to-model (T2M) setup. (a) Classification tasks are described in rich language. (b) Traditional zero-shot methods produce static representations, shared for all tasks. (c) T2M 
    generates \textit{task-specific representations and classifiers}. This allows T2M to extract task-specific discriminative features.} 
    \label{fig:fig1abc}
\end{figure}

We explore the challenge of zero-shot image classification by leveraging text descriptions. This approach pushes the boundaries of conventional classification methods by demanding that models categorize images into specific classes based solely on written descriptions, without having previously encountered these classes during training.\footnote{We note that our definitions of ``zero shot'' or ``zero shot learning'' are slightly different than the ones used in the context of text-only language models.}
In various domains, numerous attempts have been made to achieve zero-shot classification capacity (\S\ref{sec:related-work}). Unfortunately, as we now explain, existing studies are limited in two major ways: (1) Query-dependence; and (2) Richness of Language description.

First, \textit{Query-dependence}. To illustrate the issue, consider a popular family of zero-shot learning (ZSL) approaches,  which maps text (like class labels) and images to a shared space
\cite{globerson2004euclidean,zhang2015zero, zhang2017learning,sung2018learning, Pahde_2021_WACV}. 
To classify a new image from an unseen class, one finds the closest class label in the shared space. The problem with this family of shared-space approaches is that the learned representation (and the kNN classifier that it induces) remains "frozen" after training, and is not tuned to the classification task given at inference time. For instance, \textit{furry toys} would be mapped to the same shared representation regardless of whether they are to be distinguished from other \textit{toys}, or from other \textit{furry things} (see Figure \ref{fig:fig1abc}). 
The same limitation also hinders another family of ZSL approaches, which synthesize samples from unseen classes at inference time using conditional generative models, and use these samples with kNN classification \cite{elhoseiny2019creativity, jha2021imaginative}. Some approaches address the query-dependence limitation by assuming that test descriptions are known during training \cite{han2021contrastive, schonfeld2019generalized}, or by (costly) training a classifier or generator at inference time \cite{xian2018feature, schonfeld2019generalized}.  Instead, here we learn a model that produces task-dependent classifiers and representations without test-time training.

\begin{table*}[t!]
    \centering
    \scalebox{0.85}{
    \setlength{\tabcolsep}{3pt} %

\small
\begin{tabular}{|c|c|c|l|}
    \toprule 
    \textbf{Dataset}   &  \textbf{Sample} &  \textbf{Description } & \textbf{Example } 
    \\ 
    \textbf{name and type}   &  \textbf{ data } &  \textbf{ type} & \textbf{description} 
    \\ \midrule
    
    \multirow{4}{*}{ \begin{tabular}{c} ~ \\ AwA \cite{lampert2009learning-AwA}  \\ Animal \\ images \end{tabular}} & \multirow{4}{*}{\begin{tabular}{c}  \includegraphics[width=0.3\textwidth]{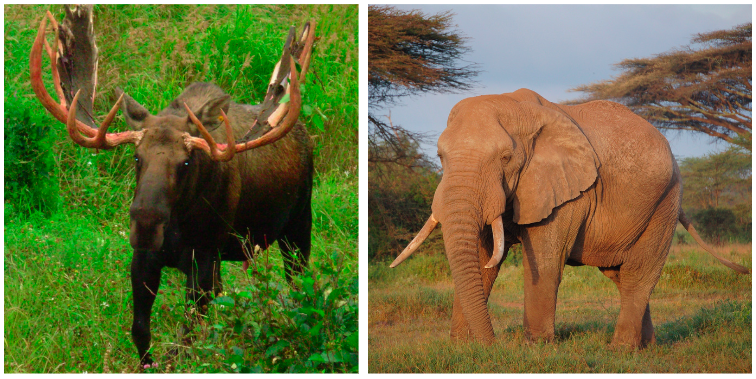} \end{tabular}} &  Class name  
    & \begin{tabular}{l}(1) \textit{Moose} \\ (2) \textit{Elephant} \end{tabular}
    \\  \cline{3-4}
    &  &\begin{tabular}{c}Long \end{tabular}  
    & \begin{tabular}{l}(1) \textit{``An animal of the deer family with humped} \\ \textit{ shoulders, long legs, and a large head with antlers.''},\\ (2) \textit{``A plant-eating mammal with a long trunk,}\\ \textit{large ears, and thick, grey skin.''} \end{tabular} 
    \\ \cline{3-4}
    &  & \begin{tabular}{c}Negative \end{tabular} & \begin{tabular}{l} (1) \textit{``An animal without stripes and not gray''},\\ (2) \textit{``An animal without fur and without horns''}\end{tabular} 
    \\ \cline{3-4}
    &  & \begin{tabular}{c} Attribute \end{tabular} & \begin{tabular}{l} (1) \textit{``Animals with fur''} \\ (2) \textit{``Animals with long trunk''} \end{tabular} 
    
    \\ \bottomrule
    
    \end{tabular}
    }
    \caption{An illustration depicting the diverse tasks within the AwA dataset is provided. Appendix~\ref{app:overview-table} contains illustrations for the remaining datasets.}
    
    \label{tab:datasets-short}
\end{table*}

The second limitation is \textit{language richness}. Natural language can be used to describe classes in complex ways. Most notably, people use negative terms, like "dogs without fur", to distinguish class members from other items. Previous work could only handle limited richness of language descriptions. For instance, it cannot represent adequately textual descriptions with negative terms \cite{akata2015evaluation, xie2021vman, xie2021cross, elhoseiny2019creativity, jha2021imaginative}. In this paper, we wish to handle the inherent linguistic richness of natural language. 

An alternative approach to address zero shot image recognition tasks involves leveraging large generative vision and language models (e.g., GPT4Vision). These foundational models, trained on extensive datasets, exhibit high performance in zero and few-shot scenarios. However, these models are associated with certain limitations: (1) They entail significant computational expenses in both training and inference. (2) Their training is specific to particular domains (e.g., vision and language) and may not extend seamlessly to other modalities (e.g., 3D data and language). (3) Remarkably, even state-of-the-art foundational models encounter challenges when confronted with tasks involving uncommon descriptions, as demonstrated in \S\ref{sec:negatives}.

In addition to the limitations posed by large generative models, there is a growing demand for smaller, more efficient models that can run on edge devices with limited computational power, such as mobile phones, embedded systems, or drones. Giant models that require cloud-based infrastructures are often computationally expensive and not suitable for real-time, on-device applications. Furthermore, some companies are unable to rely on cloud computing due to privacy concerns or legal regulations that mandate keeping sensitive user data within their local networks (on-premises). Our approach addresses these needs by enabling the automatic generation of task-specific models that are lightweight and capable of running on weaker devices without requiring cloud resources.

Here, we describe a novel deep network architecture and a learning workflow that addresses these two aspects: (1) generating a discriminative model tuned to requested classes at query time and (2) supporting rich language and negative terms. 

To achieve these properties, we propose an approach based on hypernetworks (HNs) \cite{ha2016hypernetworks}. An HN is a deep network that emits the weights of another deep network (see Figure \ref{fig:architecture} for an illustration). Here, the HN receives a set of class descriptions and emits a multi-class model that can classify images according to these classes. 
Interestingly, this text-image ZSL setup has an important symmetric structure. 
In essence, if the order of input descriptions is permuted, one would expect the same classifiers to be produced, reflecting the same permutation applied to the outputs.
This property is called \textit{equivariance}, and it can be leveraged to design better architectures \cite{finzi2020generalizing, cohen2019general, kondor2018generalization, finzi2021practical}. Taking invariance and equivariance into account has been shown to provide significant benefits for learning in spaces with symmetries like sets \cite{zaheer2017deep, maron2020learning-DSS,amosy2024late} graphs \cite{herzig2018mapping,wu2020comprehensive} and deep weight spaces \cite{navon2023}.
In general, however, HNs are not always permutation equivariant. 
We design invariant and equivariant layers and describe an HN architecture that respects the symmetries of the problem, and term it T2M-HN:  \textit{a text-to-model hypernetwork}.

\begin{figure*}[t]
    \centering
    \includegraphics[width=0.90\linewidth]{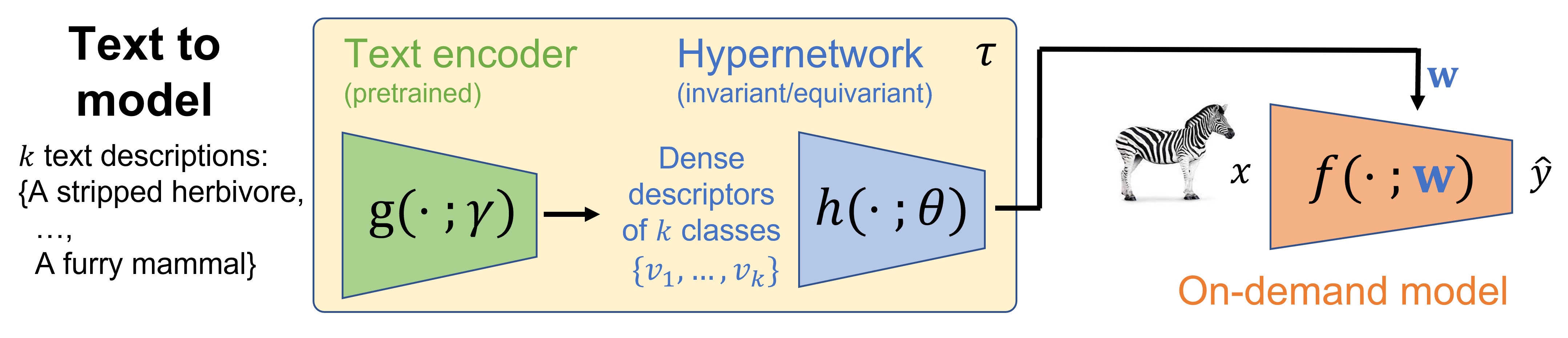}
    \caption{The text-to-model learning problem and our architecture. Our model (yellow box) receives a set of class descriptions as input and outputs weights $w$ for a downstream on-demand model (orange).  The model has two main blocks: A pretrained text encoder and a hypernetwork that obeys certain invariance and equivariance symmetries. The hypernetwork receives a set of dense descriptors to produce weights for the on-demand model. } 
    \label{fig:architecture}
\end{figure*}

We put the versatility of T2M-HN to the test across an array of zero-shot classification tasks, spanning diverse data types including images, 3D point clouds, and 3D skeletal data for action recognition. Our framework exhibits a remarkable ability to incorporate various forms of class descriptions including long and short texts, as well as class names. Notably, T2M-HN surpasses the performance of previous state-of-the-art methods in all of these setups.

Our paper offers four key contributions: (1) It identifies limitations of existing ZSL methods that rely on fixed representations and distance-based classifiers for text and image data. It proposes task-dependent representations as an alternative; (2) It introduces the Text-to-Model (T2M) approach for generating deep classification models from textual descriptions; (3) It investigates equivariance and invariance properties of T2M models and designs T2M-HN, an architecture based on HNs that adheres to the symmetries of the learning problem; and (4) It shows T2M-HN's success in a range of zero-shot tasks, including image and point-cloud classification and action recognition, using diverse text descriptions, surpassing current leading methods in all tasks.

\section{Related work}
\label{sec:related-work}
In this section we cover previous approaches to leverage textual description to classify images of unseen classes.

\textbf{Zero-shot learning (ZSL).}
The core challenge in ZSL lies in recognizing images of unseen classes based on their semantic associations with seen classes. This association is sometimes learned using human-annotated attributes \cite{li2019rethinking, song2018transductive, morgado2017semantically, annadani2018preserving}. 
Another source of information for learning semantic associations is to use textual descriptions. Three main sources were used in the literature to obtain text descriptions of classes: (1) Using class names as descriptions \cite{zhang2017learning, frome2013devise, changpinyo2017predicting,cheraghian2022zero-PCZS}; (2) using encyclopedia articles that describe the class \cite{lei2015predicting, elhoseiny2017link, qin2020generative, bujwid2021large, paz2020zest, zhu2018generative}; and (3) providing per-image descriptions manually annotated by domain experts \cite{reed2016learning,patterson2012sun,wah2011caltech-CUB}. These can then be aggregated into class-level descriptions. 

\textbf{Shared space ZSL.}
One popular approach to ZSL is to learn a joint visual-semantic representation, using either attributes or natural text descriptions. 
Some studies project visual features onto the textual space \cite{frome2013devise, lampert2013attribute, xie2021vman}, others learn a mapping from a textual to a visual space \cite{zhang2017learning, Pahde_2021_WACV}, and some project both images and texts into a new shared space \cite{akata2015evaluation,atzmon2018probabilistic,sung2018learning, zhang2015zero,atzmon2019adaptive,atzmon2020causal, samuel2021generalized,xie2021cross, radford2021learning-CLIP}. Once both image and text can be encoded in the same space, classifying an image from a new class can be achieved without further training by first encoding the image and then selecting the nearest class in the shared space. 
In comparison, instead of nearest-neighbour based classification, our approach is learned in a discriminative way.

\textbf{Generation-based ZSL.} 
Another line of ZSL studies uses generative models like GANs to generate representations of samples from unseen classes \cite{elhoseiny2019creativity, jha2021imaginative}. 
Such generative approaches have been applied in two settings. Some studies assume they have access to test-class descriptions (attributes or text) during model training. Hence, they can train a classifier over test-class images, generated by leveraging the test-class descriptions \cite{liu2018generalized,schonfeld2019generalized,han2021contrastive}. 
Other studies assume access to test-class descriptions only at test time. Hence, they map the test-class descriptions to the shared space of training classes and apply a nearest-neighbor inference mechanism. In this work, we assume that any information about test classes is only available at test time. As a result, ZSL methods assuming train-time access to information about the test classes are beyond our scope.\footnote{While these algorithms could in principle be re-trained when new classes are presented at test-time (e.g. in a continual learning \cite{ring1994continual} setup), this would result in costly and inefficient inference mechanism, and possibly also in catastrophic forgetting \cite{mccloskey1989catastrophic}. We hence do not include them in our experiments.} Yet, works assuming only test-time access to test-class information form some of our baselines \cite{elhoseiny2019creativity, jha2021imaginative}.

\textbf{Hypernetworks} (HNs, \citet{ha2016hypernetworks}) were applied to many computer vision and NLP problems, including 
ZSL~\cite{yin2022sylph}, federated learning~\cite{amosy2024late}, 
domain adaptation \cite{volk2022example}, language modeling \cite{suarez2017language}, machine translation \cite{platanios2018contextual} and many more. 
Here we use HNs for text-based ZSL. 
The work by \citet{lei2015predicting} also predicts model weights from textual descriptions, but differs in two key ways. (1) They learn a constant representation of each class; our method uses the context of all the classes in a task to predict data representation. (2) They predict weights of a linear architecture; our T2M-HN applies to deeper ones. 

\textbf{Large vision-language models (LVLM)} 
CLIP \cite{radford2021learning-CLIP}, BLIP2 \cite{li2023blip} and GPT4Vision show remarkable zero-shot capabilities for vision-and-language tasks.
A key difference between those approaches and this paper is that CLIP and BLIP2 (the training approach of GPT4Vision remains undisclosed) were trained on \textit{massive multimodal data}. In contrast, our approach leverages the semantic compositionality of \textit{language models}, without requiring paired image-text data. 
Furthermore, such large models are costly in both training and inference. They demand substantial resources, time and specialized knowledge that is not accessible to most of the research community.
We successfully applied T2M-HN in domains lacking large multimodal data, such as 3D point cloud object recognition and skeleton sequence action recognition.
The drawback is that the T2M-HN representation might react to language differences that don't matter for visual tasks. 
\section{Problem formulation}
\label{sec:formulation}

Our objective is to learn a mapping $\tau$ from a set of $k$ natural language descriptions into the space of a $k$-class image classifier. Here, we address the case where the architecture of the downstream classifier is fixed and given in advance, but this assumption can be relaxed as in \citet{litany2022federated}.

Formally, let $S^k = \{s_1,\ldots, s_k\}$ be a set of $k$ class descriptions drawn from a distribution $\mathcal{P}_k$, where $s_j$ is a text description of the $j^{th}$ class. 

Let $\tau$ be a model parameterized by a set of parameters $\phi$. It takes the descriptors and produces a set of parameters $W$ of a $k$-class classification model $f(\cdot;W)$. Therefore, we have $\tau_{\phi}: \{s_1,\ldots s_k\} \rightarrow \reals^d$, where $d$ is the dimension of $W$, that is, the number of parameters of $f(\cdot;W)$, and we denote $W = \tau_{\phi}(S^k)$.

Let $l:\mathcal{Y}\times \mathcal{Y} \rightarrow \reals^+$ be a loss function, and let $\{\xvec_i, y_i\}_{i=1}^n$ be a labeled dataset from a distribution $\mathcal{P}$ over $\mathcal{X} \times \mathcal{Y}$. For $k$-class classification, $\mathcal{Y} = \{1,\ldots,k\}$.
We can explicitly write the loss in terms of $\phi$ as follows.  
\begin{equation}
    \begin{split}
    l\left(y_i, \hat{y_i}\right) &= l\left(y_i, f(x_i;W)\right) \\
    &= l\left(y_i,f(x_i;\tau_{\phi}(S^k)) \right). 
    \end{split}
\end{equation}
See also Figure \ref{fig:architecture} and note that $\tau = h \circ g$. The goal of T2M is to minimize 

\begin{equation}
    \begin{split}
    &\phi^* = \\
    &\argmin_\phi \mathbb{E}_{S^k\sim \mathcal{P}^k} \mathbb{E}_{(x,y)\sim\mathcal{P}} \left[ l\left(y,f(x;\tau_{\phi}(S^k)) \right)\right].
    \end{split}
\end{equation}

The training objective becomes 
\begin{equation}
    \phi^* \!=\!\argmin_\phi \sum_j \sum_i l\left(y_i,f(x_i;\tau_{\phi}(S^{k_j})) \right), 
\end{equation}
where the sum over $j$ means summing over all descriptions from all sets in the training set.

\section{Our approach}
\label{sec:method}
We first describe our approach, based on HNs. We then discuss the symmetries of the problem, and an architecture that can leverage these symmetries. 

We propose to address the T2M problem, using an HNs. 
An HN is a model that outputs the weights of another model \cite{ha2016hypernetworks}. In our case, it receives a set of textual descriptions of classes to be recognized, and outputs the weights of a classifier that can discriminate them. Figure \ref{fig:architecture} illustrates our architecture. It has two components. First, a text encoder $g$ takes natural language descriptions and transforms them into dense descriptors; and second, an HN $h$  takes these dense descriptors and emits weights for a downstream classifier. 
In this paper, we do not impose any special properties on the text encoder $g$. It can be any model trained using language data (no need for multi-modal data).

\subsection{Symmetries of the T2M problem}  
Interestingly, the T2M setup imposes certain invariance and equivariance properties. Design an architecture that takes them into account can improve generalization. We now discuss these properties and then derive an architecture that captures them.

\paragraph{Equivariance properties of the classifier layer.} 
As an illustrative example, consider a downstream multi-class classifier $f_1$, that is designed to distinguish \textit{cats} from \textit{dogs}, and another classifier $f_2$, designed to distinguish \textit{dogs} from \textit{cats}. Intuitively, at the optimum, the two classifiers should be identical except for a switch of two weight vectors at the last layer  ($w_1$ in $f_1$ equal to $w_2$ in $f_2$). This has an important implication for the hypernetwork. Any permutation applied to its input class descriptions should be reflected in a parallel ordering of the weight vectors that it produces.

Consider a downstream multiclass deep classifier whose last (classification) layer has a weight vector $w_i \in \reals^m$ for the output class $i$. 
The weight matrix of the last layer is $W_{last}=\{w_1,\dots, w_k\}$ (See Figure \ref{fig:equivariance}a). 

The HN receives $k$ class descriptors and outputs their corresponding weights 
\begin{equation}
\begin{split}
    W_{last} &= \{w_1,\ldots,w_k\} \\
    &=R_{last}(\tau_{\theta}\left(\{s_1,\ldots,s_k\}\right)),
\end{split}
\end{equation}
where $R_{last}$ is a function that takes the output of $\tau$ and resizes the last $k*m$ elements to the matrix $W_{last}$. 
If the input descriptions are permuted by a permutation $\mathcal{P}$ the columns of the last layer weight should be permuted accordingly: 
\begin{equation}
    \mathcal{P}(f(x;\tau_{\phi}(S^k)) = f(x;\tau_{\phi}(\mathcal{P}(S^k)).
\end{equation}
This is the equivariant property, and the HN should obey it.

\paragraph{Invariance properties of intermediate layers.} 

Considering now the layer of the downstream classifier before the last layer ($w_d$ in Figure \ref{fig:equivariance}a). We now show that using an equivariant transformation for the last layer and an invariant transformation for the penultimate layer is sufficient to ensure that the downstream classifier is equivariant to permutation over the descriptions. A similar argument holds for earlier (lower) intermediate layers.

\begin{theorem}
Let $f$ be a two-layer neural network $f(x)=W^{last}\sigma(W^{pen} x)$, whose weights are predicted by $\tau$ $[W^{last}, W^{pen}] =  \tau(S^k)$. If $\tau(S^k)$ is equivariant to a permutation $\mathcal{P}$ with respect to $W^{last}$, and invariant to $\mathcal{P}$ with respect to $W^{pen}$, then $f(x)$ is equivariant to $\mathcal{P}$ with respect to the input of $\tau(S^k)$.
See a formal proof in the Supplemental Section \ref{app:eviv-layers}.
\end{theorem}

\subsection{Invariant and equivariant Architectures.} 
Given the equivariance property discussed above, we wish to design a deep architecture that adheres to those symmetries. To ensure that certain elements remain invariant permutation, they should be processed with a shared set of parameters \cite{wood1996representation, ravanbakhsh2017equivariance, maron2020learning-DSS}. In our case, we need to share the parameters that process input descriptions, so the model is equivariant to permutations of those inputs.

Figure \ref{fig:equivariance}(a) gives the high-level structure of the equivariant architecture of T2M-HN. Figure \ref{fig:equivariance}(b) shows the architecture of our equivariant layers. All inputs are fed into the same fully connected layer (vertical stripes). To take into account the context of each input, we sum all the inputs to obtain a context vector. We fed the context vector to a different fully connected layer (diagonal stripes) and add it to each one of the processed inputs. The invariant layer has a similar architecture (Figure \ref{fig:equivariance}(c)), but with additional summation over all equivariant outputs and another different fully connected layer (horizontal stripes).

Our HN uses several equivariant layers to process the input descriptions. We then use one prediction head for each layer of the output model. The last layer should be equivariant, so we use an equivariant prediction head. For the hidden layers, we use invariant layers (See Figure \ref{fig:equivariance}(a)).

\begin{figure*}[h]
    \hspace{0.25\linewidth} (a) \hspace{0.3\linewidth} (b) \hspace{0.25\linewidth} (c) \hspace{0.00\linewidth} 

    \centering
    \includegraphics[width=0.99\linewidth]{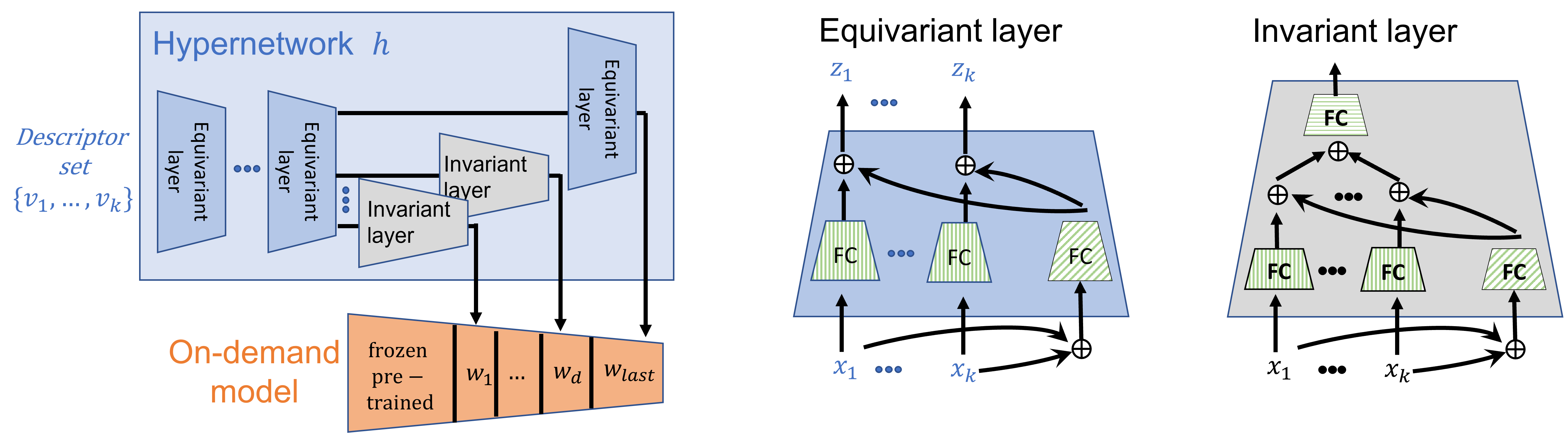}
    \caption{(a) The T2M-HN architecture for equivariant-invariant hypernetwork. The input is processed by  equivariant layers, followed by a prediction head for each layer of the target on-demand classifier $f$. The prediction head for $W_{last}$ is equivariant. Heads for earlier layers of $f$, $w_1, ...w_k$ are  invariant. (b) An architecture for the equivariant layer. Every input is processed by a fully connected (FC) layer in a Siamese manner (shared weights). Inputs are also summed and processed by a second FC layer, whose output is added back to each output. (c) An architecture for an invariant layer, following a similar structure to b.} 
    \label{fig:equivariance}
\end{figure*}

\section{Experiments}
\label{sec:experiments}
The T2M setup is about producing a model that can be applied to data from new classes. Accordingly, the model trains on data from a set of training classes, alongside their text descriptions. Then, it is tested on data from new classes, given the text descriptions of these classes. 

We evaluate T2M-HN in zero-shot classification, using
three image datasets, one 3D point cloud dataset, and one action recognition dataset. We consider various forms of text description, including single-word class labels, few-word class names, and longer descriptions that could also include negative properties (i.e. properties that the images in the class do not have). Finally, we study one-class classification based on text attributes. 
Due to space constraints, we provide a concise description of our experimental settings here. Further details can be found in Appendix \ref{app:Implementation-and-architecture}. 

\begin{table*}[t]
    \begin{minipage}[c]{0.67\linewidth}
    \scalebox{0.68}{
    \begin{tabular}{l| c c c | c c c }
        & \multicolumn{3}{c}{AWA by class name}
        & \multicolumn{3}{c}{ModelNet40 by class name}\\
         \textbf{}  & Seen   & Unseen & Harmonic & Seen   & Unseen & Harmonic\\
        
        \midrule
        CLIP & $98.9 \pm 0.2$ & NA & NA & NA & NA & NA\\
        BLIP2 & $99.6 \pm 0.1$ & NA & NA & NA & NA & NA\\
        GPT4Vision & $100 \pm 0.0$ & NA & NA & NA & NA & NA\\
        
        \midrule
        DeViSE & $78.1 \pm 1.0$ & $58.9 \pm 1.4$ & $67.2 \pm 1.9$
        & $83.6 \pm 2.7$ & $58.6 \pm 3.4$ & $68.9 \pm 3$\\
        DEM & $83.1 \pm 1.6$ & $75.1 \pm 1.2$ & $78.9 \pm 2.0$
        &$86.7 \pm 2.4$ & $57.3 \pm 3.3$ & $69.0 \pm 2.8$\\
        CIZSL & $97.0 \pm 0.1 $& $74.7 \pm 3.2$ & $84.20 \pm 2.0$
        &$97.6 \pm 0.6$& $50.1 \pm 3.6$&$66.3 \pm 3.3$ \\
        GRaWD & $96.9 \pm 0.1$ & $81.6 \pm 1.9$ & $88.6 \pm 1.1$
        &$97.8 \pm 0.5$& $52.8 \pm 3.3$&$68.3 \pm 2.8$ \\
        ZSML & $96.1 \pm 1.0$ & $80.4 \pm 2.4$ & $87.5 \pm 1.5$
        & $90.2 \pm 1.5$ & $68.6 \pm 4.5$ & $77.8 \pm 3.0$\\
        T2M-HN {\footnotesize (ours)} & $\mathbf{98.9 \pm 0.1}$ & $\mathbf{87.3 \pm 0.2}$ & $\mathbf{92.7 \pm 0.1}$
        & $\mathbf{97.9 \pm 0.1}$ & $\mathbf{75.1 \pm 0.9}$ & $\mathbf{85.0 \pm 0.4}$ \\

        \bottomrule
    \end{tabular}
    }   
    \end{minipage}\hfill
    \begin{minipage}[c]{0.25\linewidth}
    \caption{\textbf{Classification by single-word class names}. Accuracy on seen and unseen classes for AWA and ModelNet-40. Values are averages and SEM across all class pairs. LVLM have encountered all unseen classes, and cannot be applied to point clouds, hence marked as NA.
    }
    \label{tab:class_name}
    \end{minipage}\hfill
    \vspace{-5pt}  
\end{table*}

\paragraph{Baselines:}
We compare our T2M-HN with five text-based zero-shot approaches for image recognition:
(1) DEVISE \cite{frome2013devise} projects images to a pre-trained language model space by adding a projection head to a pre-trained visual classification model;
(2) Deep Embedding Model (DEM) \cite{zhang2017learning-DEM} uses the visual space as the shared embedding space; 
(3) CIZSL \cite{elhoseiny2019creativity} trains conditional GANs with a loss designed to generate samples from unseen classes without synthesizing unrealistic images.  At inference time, the GAN is conditioned on test descriptions, generates synthetic image representations, and test images are classified using kNN w.r.t. to the synthetic images;
(4) GRaWD \cite{jha2021imaginative} trains a conditional GAN with a loss that helps to reach regions in space that are hard to classify as seen classes; and
(5) ZSML \cite{verma2020meta} combines meta-learning with a WGAN, to generate samples from unseen classes, and use them to train a classifier at test time.
When relevant, we also computed the performance obtained when using CLIP, BLIP2, and GPT4Vision. Note that those models were trained using massively large datasets, so it is reasonable to assume they have seen all classes studied here. This is hence not zero-shot classification, and the results can be viewed as a ``skyline" value that zero-shot approaches should aim at.

\textbf{Datasets:}
We experiment with three image datasets: \textbf{(1) Animals with attributes} (AWA) \cite{lampert2009learning-AwA}; 
\textbf{(2) SUN} \cite{patterson2012sun};  and 
\textbf{(3) CUB} \cite{wah2011caltech-CUB}; a 3D point-clouds dataset: \textbf{(4) ModelNet40} \cite{wu20153d-modelnet40}; and an action recognition dataset: 
\textbf{(5) BABEL 120}\cite{babel}, containing sequences of body skeletons.

\textbf{Experimental protocol:}
We split the data in two dimensions: Classes and samples. For standardized comparisons the splitting classes into \textit{seen classes} used for training and \textit{unseen classes} used in evaluation. 
For each seen class we split out a set of evaluation images that are not presented during training, and used to evaluate the model on the seen classes. We stress that  "Seen" in our tables means \textit{novel images} from \textit{seen classes}. 

\textbf{Workflow:}
When training the whole architecture, we split the train seen classes. 80\% of the classes were used for training the backbone. 
Then, we froze the weights of the backbone and use the remaining 20\% to train the HN. This way, the HN learns to generalize to new classes. 
Finally, we evaluate the entire architecture on the evaluation split of the seen classes, and on the unseen classes.
At test time, the model receives $k$ class descriptions and predicts a model to classify images drawn from the corresponding $k$ classes. Unless otherwise specified, we experiment with the value of $k=2$.

\subsection{ZSL using class names: Images and 3D point clouds }
\label{sec:zs}

In the following experiment, we evaluate T2M-HN under two tasks: Zero-shot image classification and zero-shot 3D point clouds classification. We use single-word class names for both tasks as the textual class descriptions. 

\textbf{Results:} 
Table \ref{tab:class_name} shows our model reaches the highest accuracy in both experimental setups and datasets.

\subsection{ZSL using text descriptions: Images and sequences of 3D skeletons}
\label{sec:zs_des}

Next, we evaluate T2M-HN when using richer text descriptions: \textbf{(1) For SUN}, we use short class descriptions provided by the original dataset. Specifically, SUN includes many multi-word class names like ``parking garage indoor'' or ``control tower outdoor''. 
\textbf{(2) For BABEL 120} we use the action names provided by the original dataset. Many of the actions have multi-word, descriptive names such as \textit{``take of bag''}.
\textbf{(3) For AwA}, we use synthetic class descriptions generated by a GPT model. See detailed examples in the Appendix \ref{app:gpt-examples}. We will publish the full set of descriptions for reproducibility. 
\textbf{(4) For CUB}, we use the descriptions of each image in a given class as a possible description of the class. 

In the CUB dataset, bird species from the same taxonomic family are harder to distinguish from each other than random pairs of species \cite{vedantam2017context}. We used the Datazone dataset of bird species \cite{birdlife2022} and annotated each species with its corresponding taxonomic family. Based on this information, we defined pairs of bird species from two different families as \textit{easy} and pairs from the same family as \textit{hard}.

\begin{table*}[h!]
    \setlength{\tabcolsep}{2pt}
    \centering
    \scalebox{0.77}{
    \begin{tabular}{l| c c c | c c c | c c c} 
        & \multicolumn{3}{c}{SUN by short description}
        & \multicolumn{3}{c}{BABEL by short descriptions }
        & \multicolumn{3}{c}{AWA by GPT descriptions }
        \\
         \textbf{}  & Seen   & Unseen & Harmonic & Seen   & Unseen & Harmonic & Seen & Unseen & Harmonic\\
                 \midrule
        CLIP & $99.1 \pm 0.4$ & NA & NA
        &NA&NA&NA
        & $93.7 \pm 0.2$ & NA & NA\\

        BLIP2 & $98.9 \pm 0.1$ & NA & NA
        &NA&NA&NA
        & $92.1 \pm 0.4$ & NA & NA\\

        GPT4Vision & $99.8 \pm 0.2$ & NA & NA
        &NA&NA&NA
        & $97.5 \pm 0.5$ & NA & NA\\

        \midrule
        DeViSE &$52.0 \pm 1.4$& $58.9 \pm 1.1$& $55.2 \pm 0.9$ 
        &$65.9 \pm 4.4$& $51.1 \pm 2.0$& $57.6 \pm 2.8$ 
        & $91.8 \pm 1.6$ & $70.0 \pm 3.7$ & $79.4 \pm 2.2$  \\
        DEM & $83.2 \pm 1.1$ & $83.2 \pm 1.4$ & $83.2 \pm 0.9$
        & $56.6 \pm 2.4$ & $50.2 \pm 1.1$ & $53.2 \pm 1.5$
        & $93.9 \pm 1.2$ & $73.0 \pm 3.3$ & $82.1 \pm 1.8$  \\

        CIZSL  &$94.0 \pm 0.1$& $80.3 \pm 0.6$& $86.6 \pm 0.3$ &
        $82.7 \pm 2.1$& $62.5 \pm 1.3$& $71.2 \pm 1.2$ &
        $96.6 \pm 0.1 $& $80.7 \pm 2.2$ & $87.9 \pm 1.3$\\
        GRaWD & $95.5 \pm 0.1$& $84.7 \pm 0.5$& $89.8 \pm 0.3$ &
        $83.7 \pm 1.8$& $62.2 \pm 1.1$& $71.3 \pm 1.0$ &
        $96.8 \pm 0.1$ & $81.1 \pm 0.2$ & $88.3 \pm 1.2$ \\
        ZSML & $\mathbf{96.9 \pm 0.1}$& $85.5 \pm 0.4$& $90.8 \pm 0.2$ 
        &$52.6 \pm 1.3$ & $51.2 \pm 0.9$&$51.9 \pm 1.1$ 
        & $97.4 \pm 0.5$ & $72.3 \pm 2.7$ & $82.9 \pm 1.8$  \\
        T2M-HN {\footnotesize (ours)} & $95.8 \pm 0.1$& $\mathbf{88.4 \pm 0.1}$& 
        $\mathbf{92.0 \pm 0.1}$
        & $\mathbf{95.3 \pm 0.1}$& $\mathbf{77.6 \pm 0.1}$& 
        $\mathbf{85.5 \pm 0.1}$
        & $\mathbf{98.7 \pm 0.1}$&  $\mathbf{83.3 \pm 0.1}$& $\mathbf{90.3 \pm 0.1}$ \\
        \bottomrule
    \end{tabular}
    }
    \caption{\textbf{Classification using short and rich class descriptions}. Values are the mean ($\pm$ s.e.m) accuracy 
    averaged over 100 random class pairs (for SUN and BABEL 120) and all class pairs (for AwA). LVLM have encountered all unseen classes, and cannot be applied to 3D skeletons, hence marked as NA.
    }
    \label{tab:descriptions}
\end{table*}

\begin{figure}[!ht]
    \centering
    \includegraphics[trim={5cm 4cm 6cm 0cm},clip, width=0.48\textwidth]{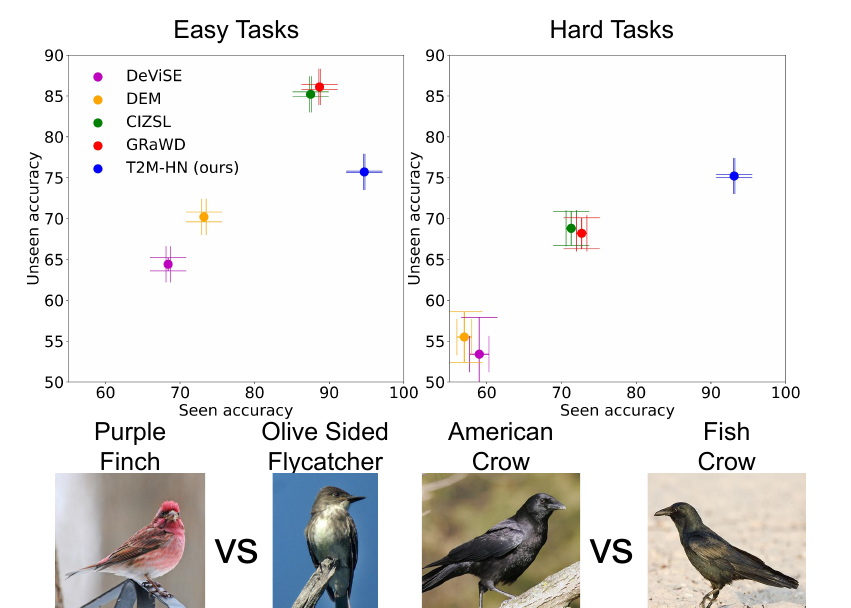}
    \caption{\textbf{Classifying easy and hard pairs of bird species from the CUB dataset}. Easy tasks involve binary classification of bird pairs from different taxonomy families. Hard tasks classify bird pairs within the same taxonomy  family. Mean accuracy is shown for images from both seen (x-axis) and unseen (y-axis) classes, averaged across all pairs.}
    \label{fig:cub_results} 
\end{figure}

\textbf{Results:} 
Table \ref{tab:descriptions} presents the classification accuracy obtained using class descriptions, for the AWA, SUN, and BABLE datasets. T2M-HN outperforms all baselines.
Figure \ref{fig:cub_results} shows the results for the CUB dataset with easy and hard tasks.
To better understand the results, consider an important distinction between our approach and previous  \textit{shared-representation} approaches. These approaches aim to learn class representations that would generalize to new classification tasks. In contrast, our approach aims to build task-specific representations and classifiers. For easy tasks, task-dependent representation may not be important because the input contains a sufficient signal for accurate classification. In contrast, in hard tasks, a model would benefit from task-dependent representation to focus on the few existing discriminative features of the input examples. Indeed, as demonstrated in Figure \ref{fig:cub_results}, in the easy tasks, although our model is superior on the seen classes, it is outperformed by the GAN-based baselines on unseen classes. In contrast, for the hard tasks, where task-specific class representation is more valuable, our model is superior on both seen and unseen classes.

\begin{figure}[t!]
    \centering
    \includegraphics[width=0.75\linewidth]{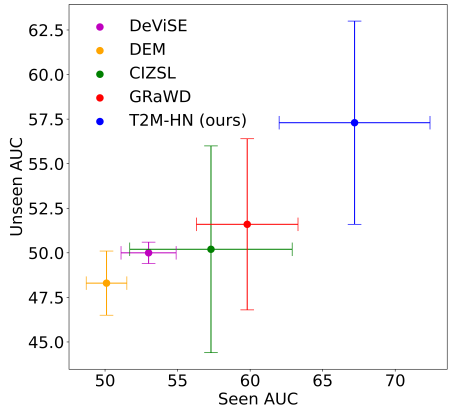}
    \caption{AUC of seen and unseen classes, in a one class task that crosses species boundaries: "\textit{Animals that have horns}". Shown are averages over 53 attributes.
    }
    \label{fig:one_class} 
\end{figure}

\begin{table*}[t]
    \begin{minipage}[c]{0.75\linewidth}
    \scalebox{0.78}{
    \begin{tabular}{l| c c c | c c c } 
        & \multicolumn{3}{c}{Negative descriptions}
        & \multicolumn{3}{c}{Negative \& positive descriptions}\\
         \textbf{AWA data}  & Seen   & Unseen & Harmonic & Seen   & Unseen & Harmonic\\
         \midrule
        CLIP & $19.9 \mypm 2.2$ & NA & NA & $56.8 \mypm 2.9$ & NA & NA \\
        BLIP2 & $73.9 \mypm 0.6$ & NA & NA & $50.1 \mypm 0.7$ & NA & NA \\
        GPT4Vision & $27.6 \mypm 0.9$ & NA & NA & $54.1 \mypm 0.9$ & NA & NA \\

        \midrule
         DeViSE &$57.3 \mypm 4.9$ & $ 54.5 \mypm 5.2$ & $55.9 \mypm 5.0$ &
        $ 79.5 \mypm 3.6 $ & $ 61.5 \mypm 4.5 $ & $ 69.4 \mypm 4.0 $\\
        DEM &$81.7 \mypm 1.2$&$73.7 \mypm 1.6$&$77.5 \mypm 1.0$&
        $78.2 \mypm 1.7$ & $69.1 \mypm 1.6$ & $73.4 \mypm 1.2$\\
        CIZSL  &$58.3 \mypm 0.8 $& $56.6 \mypm 3.4$ &$ 57.5 \mypm 1.8$&$93.9 \mypm 0.2$& $71.6 \mypm 2.3$&$81.2 \mypm 1.5$ \\
        GRaWD  &$54.9 \mypm 0.8$&$56.0 \mypm 3.2$ & $55.3 \mypm 1.6 $ &$95.0 \mypm 0.2$& $73.9 \mypm 2.0$&$83.2 \mypm 1.5$\\
        T2M-HN{\footnotesize(ours)} &$\mathbf{90.0 \mypm 0.2}$&$\mathbf{77.1 \mypm 0.3}$&$\mathbf{83.0 \mypm 0.2}$ 
        & $\mathbf{96.6 \mypm 0.2}$ & $\mathbf{82.9 \mypm 0.2}$ & $\mathbf{89.2 \mypm 0.1}$ \\
        
        \bottomrule
    \end{tabular}
    }   
    \end{minipage}\hfill
    \begin{minipage}[c]{0.25\linewidth}
    \caption{\textbf{Classification using negative descriptions}. Mean accuracy for images from seen and unseen AwA classes, averaged over all class pairs. LVLMs, trained on extensive datasets, likely encountered all unseen classes, hence marked as NA.}
    \label{tab:negative}
    \end{minipage}\hfill
    \vspace{-5pt}  
\end{table*}

\subsection{Descriptions with negative terms}
\label{sec:negatives}
To this point, we have assumed that the descriptions correspond to properties of the class. However, descriptions could also state which properties the class does \textbf{not} have. For example, one may want to classify animals that ``do not live in the water", or animals that ``do not fly".
To create such negative descriptions for the AwA data, we used the list of attributes provided for each class in AwA. For each class, we randomly sampled 4 attributes that do not apply to that class. 

\textbf{Results:} 
Table \ref{tab:negative} shows our findings for two scenarios: purely negative descriptions (left side) and balanced positive and negative descriptions (right side), maintaining equal training and testing ratios for both scenarios.

T2M-HN outperforms all baselines by significant gaps. Presumably, the best baseline, GRaWD, which generates image features from the textual descriptions, fails to generate proper images given negative attributes. Interestingly, LVLM performance significantly drops in these scenarios, likely because these models were trained on image captions that seldom include negative descriptions.

\subsection{Identifying complex classes membership}
\label{sec:set-of-classes}
Typically, zero-shot classification involves distinguishing ``natural categories" \cite{rosch1973natural} like \textit{``cats"} and \textit{``dogs"}. 
However, We may want to generate classifiers that follow more complex class boundaries, aggregating over multiple natural classes. For instance, 
\textit{``animals with horns"} combine several classes from a rhino to a deer. 

To test T2M-HN in this scenario, we created a set of one-class classification tasks designed to recognize images based on properties that cut through class boundaries.
To make the evaluation systematic, we used attributes from AwA, and eliminate non-visual attributes. Details of the protocol are given in Appendix \ref{app:supersets}. 
We report the average Area Under the Recall-Precision Curve over seen classes and unseen classes.

\textbf{Results:}
Figure \ref{fig:one_class} shows that T2M-HN captures the complex semantic distinctions of our task better than baselines. We attribute this to its ability to draw new classifiers for each new description. 

\subsection{T2M-HN classifiers are task-specific}
\label{sec:grad_cam}

Leading text-based ZSL methods map class descriptions or images to a shared representation, but that mapping is constant for all classification tasks.
Our T2M-HN is designed to use information about the classes of each specific classification task.

We use GradCam \cite{selvaraju2017grad} and examine what image areas are used in different classification tasks. Figure \ref{fig:gradcam} explores two such examples. The upper three panels show the image regions that are used for classifying the image as a \textit{Dolphin}.
When classifying dolphin vs. deer, the model gives most of its weight to the background (ocean water and waves), which is reasonable since an image of a deer probably will not contain those elements in the background. However, when classifying dolphin vs. killer whale, the model gives most of its weight to the dolphin itself, since the background of a dolphin image may be similar to the background of a whale image.

\begin{figure}[t!]
    \includegraphics[width=0.95\linewidth]{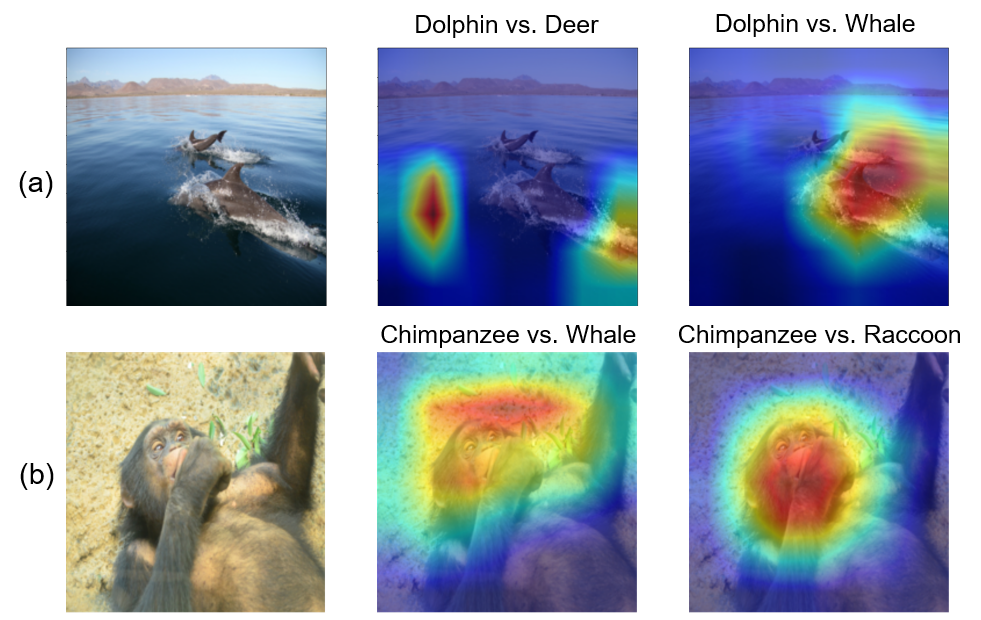} 
    \caption{Class context affects the predicted classifier. \textbf{Top left:} An image of a dolphin. \textbf{Top middle:} gradcam heat map when classifying the dolphin image using a model trained for \textit{dolphin} vs \textit{deer}: The model is strongly affected by the background ocean water, presumably because the negative class lives on land. \textbf{Top right:} Recognition using a model for \textit{dolphin} vs. \textit{killer whale}: the model attends to the dolphin, since background would be similar for both classes. \textbf{Bottom:} A similar effect for a chimpanzee.}
    \label{fig:gradcam}
\end{figure}

\section{Conclusion}
We introduced the T2M learning algorithm, a novel approach that generates an image recognition classifier ``on demand'' using only class descriptions provided at test time. T2M allows for task-dependent class representations rather than fixed ones. We analyzed the group symmetries a T2M model must adhere to and introduced T2M-HN, a model based on HNs that obeys these symmetries. Through extensive experiments across various classification scenarios—including images, 3D point clouds, and action recognition—we explored the adaptability of the model to descriptions of differing complexities, from single and few-word class names, through long text descriptions, all the way to ``negative" and attribute descriptions. Our results clearly demonstrate the potential of the T2M modeling approach.

\section{Limitations}
\textbf{Non-Visual descriptions} 
Our objective is to classify images belonging to previously unseen classes by leveraging textual descriptions. Nevertheless, it is noteworthy that textual descriptions may occasionally encompass non-visual attributes, that may mislead the model to look for irrelevant features. Due to this potential challenge, we tested our approach in similar challenges like negative descriptions (\S\ref{sec:negatives}) and complex classes membership (\S\ref{sec:set-of-classes}). 

\textbf{Hypernetwork Training Insights} 
The proposed architecture is based on hypernetworks, which are generally considered more challenging to train efficiently than standard neural networks. For instance, training a hypernetwork is probably less stable compared to training classical convolutional neural networks. However, our inner optimization search reveals that numerous parameter combinations yield satisfactory outcomes. This success is likely attributable to the entire system utilizing a uniform supervision signal through a single cross-entropy objective.

\bibliography{cite}
\include{supp}

\end{document}

%% file: math_commands.tex
\usepackage{amsmath,amsfonts,bm}

\def\eqref#1{equation~\ref{#1}}

\def\1{\bm{1}}

\DeclareMathAlphabet{\mathsfit}{\encodingdefault}{\sfdefault}{m}{sl}
\SetMathAlphabet{\mathsfit}{bold}{\encodingdefault}{\sfdefault}{bx}{n}

\DeclareMathOperator*{\argmin}{arg\,min}

%% file: supp.tex
\appendix

\section{Overview of evaluation datasets and tasks.}
\label{app:overview-table}
We evaluate the versatility of T2M-HN across a spectrum of zero-shot classification tasks, encompassing different data types such as images, 3D point clouds, and 3D skeletal data for action recognition (see Table \ref{tab:datasets}). Our framework demonstrates a remarkable capability to assimilate diverse forms of class descriptions, including both long and short texts, as well as class names. Importantly, T2M-HN outperforms previous state-of-the-art methods in all of these experimental setups.

\begin{table*}[t!]
    \centering
    \scalebox{0.8}{
    \setlength{\tabcolsep}{3pt} %

\small
\begin{tabular}{|c|c|c|l|}
    \toprule 
    \textbf{Dataset}   &  \textbf{Sample} &  \textbf{Description } & \textbf{Example } 
    \\ 
    \textbf{name and type}   &  \textbf{ data } &  \textbf{ type} & \textbf{description} 
    \\ \midrule

    \begin{tabular}{c} ModelNet-40 \cite{wu20153d-modelnet40} \\3D Point Clouds \\ CAD models \end{tabular}
    & \begin{tabular}{c}\includegraphics[trim={4cm 4cm 3cm 4cm},clip, width=0.15\textwidth]{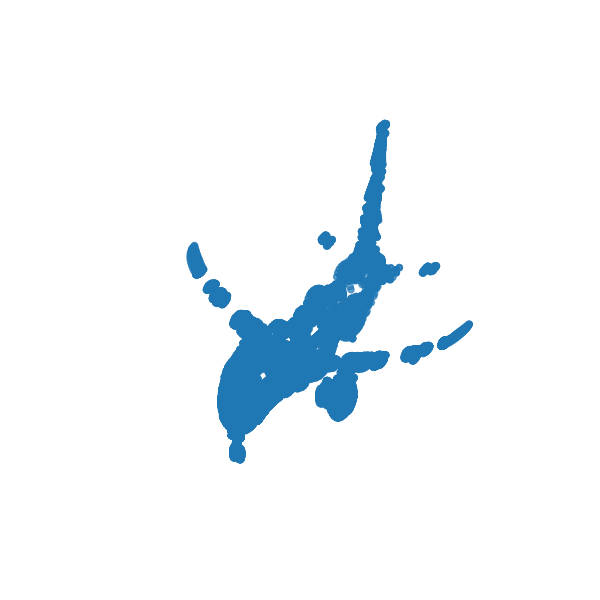} 
    \includegraphics[trim={4cm 3cm 4cm 3cm},clip, width=0.15\textwidth]{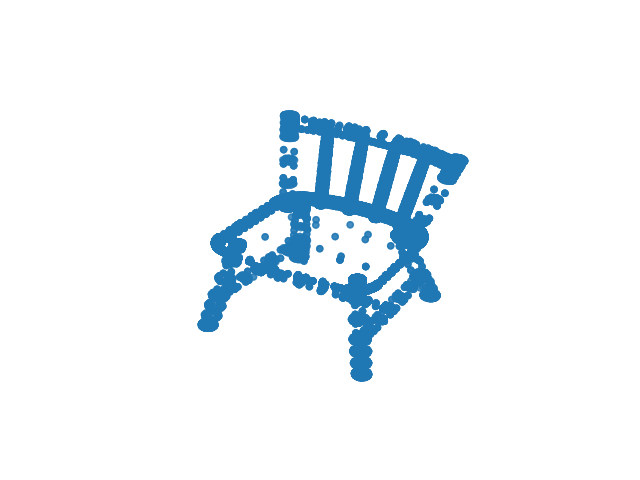}\end{tabular}
    & \begin{tabular}{c} Class \\ name\end{tabular}           
    & \begin{tabular}{l}(1) \textit{Airplane} \\ (2) \textit{Chair} \end{tabular}

    \\ \midrule 
    
    \multirow{4}{*}{ \begin{tabular}{c} ~ \\ AwA \cite{lampert2009learning-AwA}  \\ Animal \\ images \end{tabular}} & \multirow{4}{*}{\begin{tabular}{c}  \includegraphics[width=0.3\textwidth]{figs/dataset_examples/awa2.png} \end{tabular}} &  Class name  
    & \begin{tabular}{l}(1) \textit{Moose} \\ (2) \textit{Elephant} \end{tabular}
    \\  \cline{3-4}
    &  &\begin{tabular}{c}Long \end{tabular}  
    & \begin{tabular}{l}(1) \textit{``An animal of the deer family with humped} \\ \textit{ shoulders, long legs, and a large head with antlers.''},\\ (2) \textit{``A plant-eating mammal with a long trunk,}\\ \textit{large ears, and thick, grey skin.''} \end{tabular} 
    \\ \cline{3-4}
    &  & \begin{tabular}{c}Negative \end{tabular} & \begin{tabular}{l} (1) \textit{``An animal without stripes and not gray''},\\ (2) \textit{``An animal without fur and without horns''}\end{tabular} 
    \\ \cline{3-4}
    &  & \begin{tabular}{c} Attribute \end{tabular} & \begin{tabular}{l} (1) \textit{``Animals with fur''} \\ (2) \textit{``Animals with long trunk''} \end{tabular} 
    
    \\ \midrule
    
    \begin{tabular}{c} SUN  \cite{patterson2012sun} \\ Images of scenes \\ and places \end{tabular} 
    & \begin{tabular}{c} \includegraphics[width=0.3\textwidth]{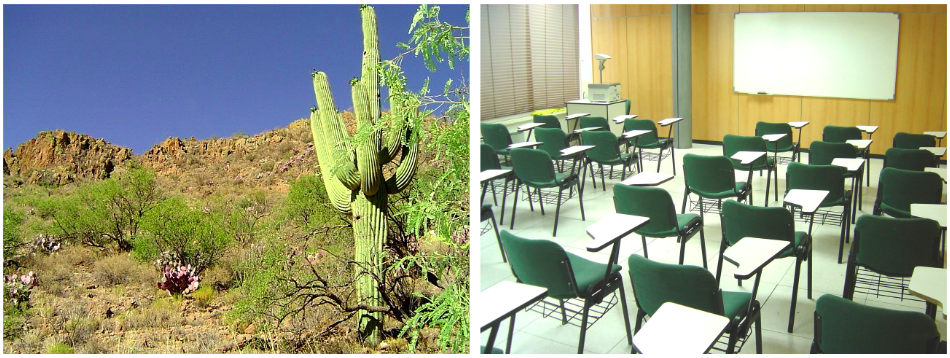} \end{tabular}
    & \begin{tabular}{c} Short \end{tabular}    
    & \begin{tabular}{l} (1) \textit{``Desert vegetation''},\\ (2) \textit{``Lecture room''}\end{tabular}
    
    \\ \midrule
    
      \begin{tabular}{c} CUB \cite{wah2011caltech-CUB} \\ Images of \\ bird species\end{tabular}
    & \begin{tabular}{c}\includegraphics[width=0.30\textwidth]{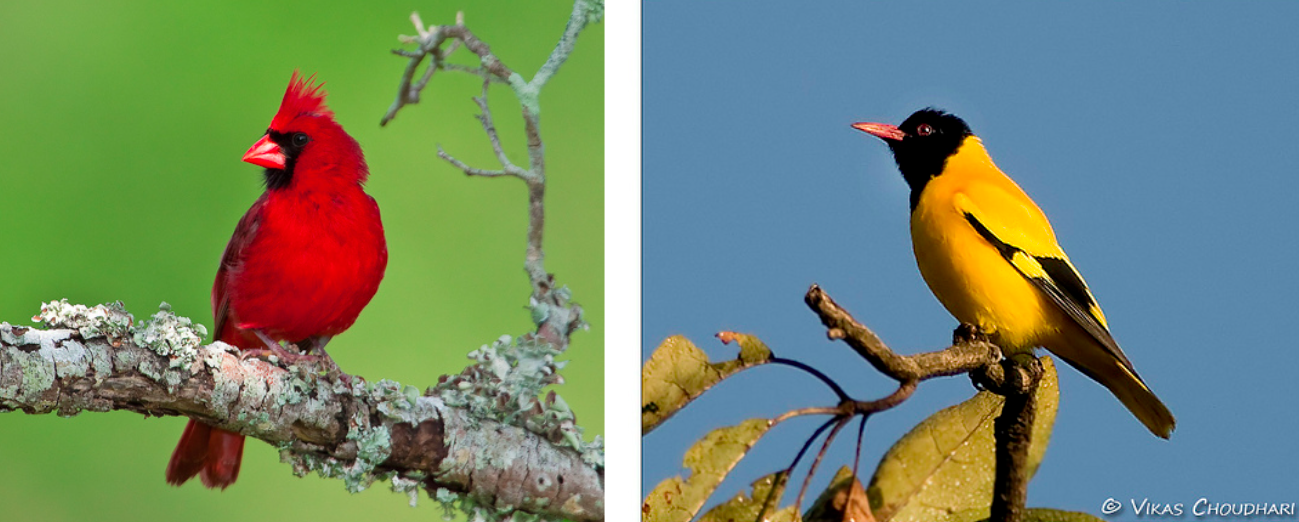}   \end{tabular}
    & \begin{tabular}{c} Long  \end{tabular}    
    & \begin{tabular}{l} (1) \textit{``This bird is red with an orange beak and black}\\ \textit{  eyes and eyebrow.''},\\ (2) \textit{``a small yellow bird with  a black } \\ \textit{chest and tail.''}\end{tabular}                      

    \\ \midrule
  
      \begin{tabular}{c} BABEL 120 \cite{babel} \\ Sequences of \\ 3D skeletal data \end{tabular} 
    & \begin{tabular}{c}\includegraphics[width=0.15\textwidth]{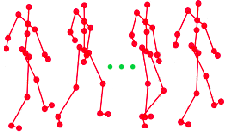} \end{tabular}
    & \begin{tabular}{c} Short \end{tabular}           
    & \begin{tabular}{l} (1) \textit{``Take off bag''}, \\ (2) \textit{``Type on a keyboard''} \end{tabular}
    
    \\ \bottomrule
    
    \end{tabular}
    }
    \caption{Overview of evaluation datasets and tasks.}
    
    \label{tab:datasets}
\end{table*}

\section{Implementation and architecture}
\label{app:Implementation-and-architecture}

\textbf{Implementation and architecture:} 
We encode single-word class names from the AwA dataset using 
Glove \cite{pennington2014glove} and longer descriptions, as well as class names, from ModelNet-40 using  SBERT \cite{reimers2019sentence-SBERT}. 
For images, the visual target model had a backbone based on a frozen ResNet-18 \cite{he2016deep-ResNet}, pretrained on ImageNet with one or two fully connected layers, predicted by the HN. For 3D point-cloud data, the backbone was PointNet \cite{qi2017pointnet}, again with one or two predicted fully-connected layers. For action recognition data, we follow \cite{babel} and use 2 stream-AGCN \cite{DBLP:conf/cvpr/ShiZCL19a}, with one or two predicted fully-connected layers as well.

For CLIP, we use the CLIP encoder followed by $k$-NN classifier in the CLIP space \cite{radford2021learning-CLIP}. For BLIP we use LoRA to tune the model to the classification task using the train split. For GPT4Vision we use the prompt to demonstrate the task, followed by the classification task from the test split.
Since we have a limited number of calls to those models we sampled classes and descriptions from the test split. We increased the sample size until the SEM was small enough to claim statistical significance.

\textbf{Experimental protocol:}
We split the data in two dimensions: Classes and samples. For standardized comparisons the splitting classes into \textit{seen classes} used for training and \textit{unseen classes} used in evaluation, follows the split used by \cite{xian2018feature} for AWA, the split of  \cite{cheraghian2022zero-PCZS} for Modelnet40 and the standard split of \cite{wah2011caltech-CUB} for CUB. Since there is no official split for SUN and BABEL, we share our random split in Section \ref{app:dataset-splits}. 
As in other ZSL protocols, for each seen class we split out a set of evaluation images that are not presented during training, and used to evaluate the model on the seen classes. For AwA, CUB, SUN and BABEL 120 we randomly selected  10\% of images for ``seen" evaluation. For ModelNet40 we use the test split  in \cite{wu20153d-modelnet40}. We stress that  "Seen" in our tables means \textit{novel images} from \textit{seen classes}. 

Our approach follows standard benchmark practices widely accepted in the field, including the use of benchmarks that may have some class overlap with ImageNet, as in many recent studies. We assess the impact of these overlaps on novelty. For ModelNet40 (point clouds) and BABEL (3D skeletal data), no pretrained ResNet is used. In the CUB dataset, only one class (Indigo Bunting) overlaps with ImageNet out of 51 unseen classes, making the effect negligible. In the SUN dataset, three classes (volcano, boathouse, palace) overlap with ImageNet out of 47 unseen classes, which we also consider negligible. However, in the AwA dataset there is significant overlap, with 5 out of 10 classes present in both ImageNet and the AwA unseen classes.
To better understand the influence on AWA results, we measure the accuracy when classifying two classes that are part of Imagenet classes, and 2 classes that are not. The results are $81.4 \pm 0.01$ and $82.2 \pm 0.09$ respectively.

\textbf{Training cost:}
Our training was completed in $\sim30$ minutes on a single 2080Ti GPU. This is faster than baselines: DEM and DEVISE require twice as long for training (1 hr), while ZSML, GRAWD and CIZSL took x4 the time (2 hrs).
This result agrees with previous literature on HNs, e.g. \cite{DBLP:conf/iclr/BrockLRW18, galanti2020modularity}.

\section{Hyperparameter optimization}
\label{app:HPO}
We tune hyperparameters using a held-out set described below. 

For the HN optimizer, we tuned the learning rate $\in \{0.001, 0.005, 0.01, 0.05, 0.1\}$,
momentum $\in \{0.1, 0.3, 0.9\}$,
weight decay $\in \{0.00001, 0.0001, 0.001, 0.1\}$, and number of HN training epochs $\in \{50,70,100\}$.

For the on-demand target model, we fixed the optimizer to have a learning rate of 0.01, momentum of 0.9 and weight decay of 0.01. We tuned the batch size $\in \{16, 32, 64, 128\}$ and the number of training epochs $\{1, 2, 3, 5, 10\}$.

We tried several sizes for the HN architecture with one hidden layer, $\{30, 50, 120, 300\}$. We also describe results with two layers in the ablation section at the supplemental Sec. \ref{app:ablation}.

Recall that we split the data across two dimensions: classes and samples. When training the backbone model, we held out 20\% of training (seen) classes for training the HN on classes the backbone does not see. From those classes, we held out images to serve as a validation set. We used those images of seen classes to evaluate the architecture performance and chose the hyperparameters based on that estimation.

\section{Equivariant and invariant layers}
\label{app:eviv-layers}

\begin{theorem}
Let $f$ be a two-layer neural network $f(x)=W^{last}\sigma(W^{pen} x)$, whose weights are predicted from descriptors $S^k = \{s_1,\ldots, s_k\}$ such that $[W^{last}, W^{pen}] =  \tau(S^k)$. If $\tau(S^k)$ is equivariant to a permutation $\mathcal{P}$ with respect to $W^{last}$, and invariant to $\mathcal{P}$ with respect to $W^{pen}$, then $f(x)$ is equivariant to $\mathcal{P}$ with respect to the input of $\tau(S^k)$.

\end{theorem}

\begin{proof}
    From the equivariance of $f(x)$ to a permutation $P$ over the input $S^k$, we have $\mathcal{P}(f(x_i;\tau_{\phi}(S^k)) = f(x_i;\tau_{\phi}(\mathcal{P}(S^k))$. Denote by $m$ the number of rows of $W^{last}$ and $z^{pen}=\sigma(W^{pen} x)$. We have 
    \begin{equation}
    \begin{split}
        \mathcal{P}(f(x;\tau_{\phi}(S^k)) & = \mathcal{P}(W^{last}\sigma(W^{pen} x)) \\
        &=\mathcal{P}(W^{last}z^{pen})) \\
        & = \mathcal{P}(
        \begin{bmatrix}
        W^{last}_1 z^{pen} \\
        . \\
        . \\
        W^{last}_m z^{pen} \\
        \end{bmatrix}
    )\\
        & =
        \begin{bmatrix}
        W^{last}_{\mathcal{P}(1)} z^{pen} \\
        . \\
        . \\
        W^{last}_{\mathcal{P}(m)} z^{pen} \\
        \end{bmatrix}
        \\&= \mathcal{P}(W^{last})z^{pen}.
    \end{split}
    \end{equation}
    If $\tau(S^k)$ is equivariant to $\mathcal{P}$ with respect to $W^{last}$, and invariant to $\mathcal{P}$ with respect to $W_{pen}$, then $\tau(\mathcal{P}(S^k)) = [\mathcal{P}(W^{last}), W^{pen}]$, so     
    \begin{equation}
    \begin{split}
        \mathcal{P}(f(x;\tau_{\phi}(S^k)) 
        &= \mathcal{P}(W^{last})z^{pen} \\
        &= f(x;\tau_{\phi}(\mathcal{P}(S^k)).
    \end{split}
    \end{equation}
\end{proof}

\section{Multi-class classification}
\label{app:triplets}
To demonstrate the flexibility of our approach to deal with multiple classes, we evaluated T2M-HN in 3-way classification tasks. In each task, the on-demand model classifies the image into one out of three classes. For example, such a task could be to classify whether an image is a dog, a cat, or an elephant.
We use the same workflow as described in Section \ref{sec:experiments}, with $k=3$.
Results are in Table \ref{tab:triplets}. T2M-HN outperforms all baselines by a large margin. 

\begin{table}[h]
    \centering
    \setlength{\tabcolsep}{4pt}
    \scalebox{0.8}{
    \begin{tabular}{ l | c c c} 
        & \multicolumn{3}{c}{AwA triplets by class name}
        \cr & Seen   & Unseen & Harmonic\\
        \midrule
        DeViSE  & $95.1 \pm 0.7$ & $55.6 \pm 3.6$ & $70.2 \pm 1.2$ \\
        DEM & $94.6 \pm 0.7$ & $64.3 \pm 3.0$ & $76.6 \pm 1.1$\\
        CIZSL   & $97.0 \pm 0.4 $& $62.0 \pm 2.9$ & $75.6 \pm 2.1$ \\
        GRaWD  & $96.4 \pm 0.5$ & $68.5 \pm 3.0$ & $80.0 \pm 2.0$ \\
        T2M-HN (ours) & $\mathbf{98.1 \pm 0.1}$&  $\mathbf{75.3 \pm 0.1}$& $\mathbf{85.2 \pm 0.1}$ \\
        \bottomrule
    \end{tabular}
    }
    \caption{Classification by class descriptions. Mean classification accuracy and SEM on images from seen and unseen classes. Averages are over 100 random class triplets}
    \label{tab:triplets}
\end{table}


\section{3D point cloud multiclass classification}
\label{app:modelnet_res}
While T2M-HN is designed to excel in binary classification, it can be easily applied to multiclass problems. For comparison with previous models we evaluate its performance in multi-class settings, where T2M-HN predicts a model that classifies all seen and unseen classes, instead of two specific classes. 
Table \ref{tab:model_net_multiclass} shows the results of this experiment. We report the result when classifying new samples from the seen classes (30-classes classification) and from the unseen classes (10-classes classification).
T2M-HN achieves SOTA results in this setup as well. It leverages the text generalization of the HN model to distinguish between unseen classes. 

We further computed the top-$k$ accuracy achieved by running T2M-HN for the unseen classes. Figure \ref{fig:top_k} plots the accuracy as a function of $k$. T2M-HN provides superior accuracy for all tested values of $k$. To calculate the top-k performance of the GAN-based models, after generating the images, we checked if any of $K$ closest neighbors of an image is of the correct class.
\begin{table}[t!]
    \setlength{\tabcolsep}{3pt}
    \centering
    \scalebox{0.99}{
    \begin{tabular}{l| c c c} 
        & \multicolumn{3}{c}{ModelNet40 by class name}\\
         \textbf{}  & Seen   & Unseen & Harmonic\\
        \midrule
        DeViSE  & $47.2$ & $14.5$ & $22.2$\\
        DEM  & $46.8$ & $7.0$ & $12.3$\\
        CIZSL & $75.6 $& $6.0$ & $11.0$ \\
        GRaWD  & $75.2$ & $10.9$ & $19.0$\\
        T2M-HN (ours) & $\mathbf{76.3}$ & $\mathbf{18.9}$ & $\mathbf{30.3}$\\
        \bottomrule
    \end{tabular}
    }
    \caption{\textbf{3D point-cloud object recognition using single-word class names}. Multiclass accuracy on seen and unseen classes for ModelNet-40. The seen accuracy is between 30 classes, and the unseen accuracy is between 10 classes. 
    }
    \label{tab:model_net_multiclass}
\end{table}

\begin{figure}[t]
    \centering
    \includegraphics[width=0.8\linewidth]{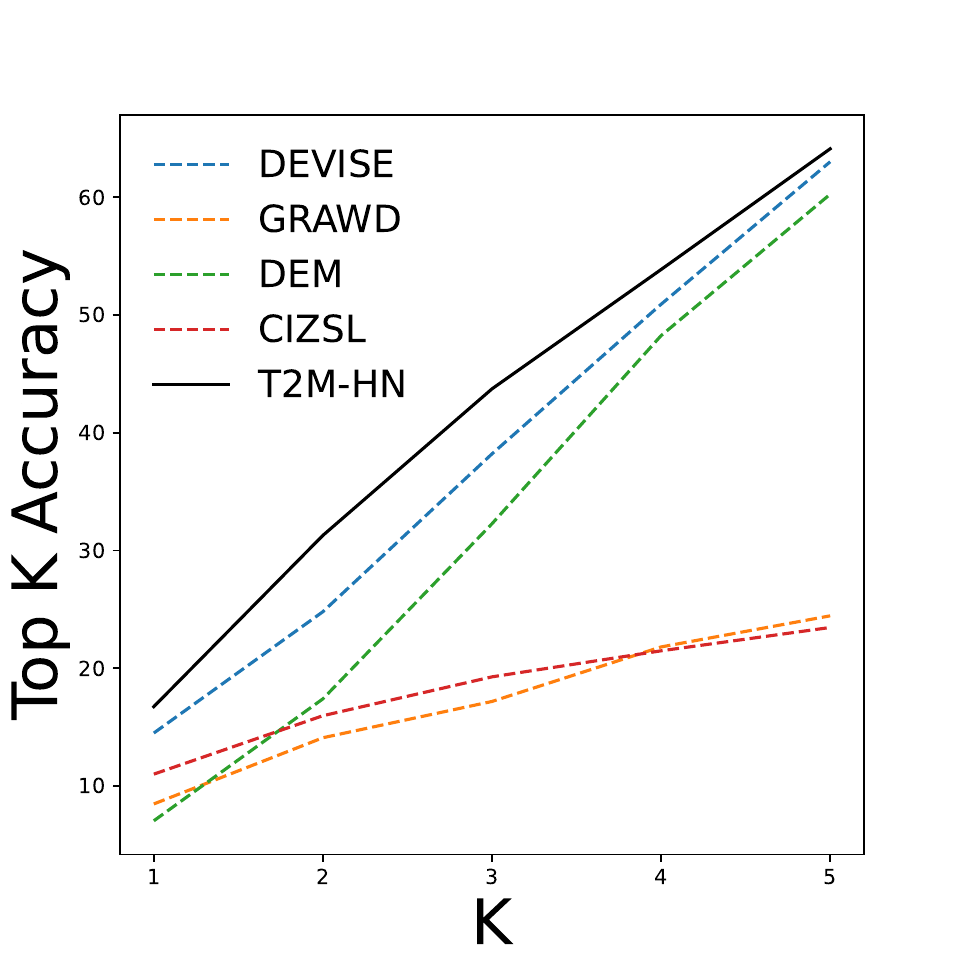}
    \caption{Accuracy at k for experiments with 3D point cloud from ModelNet-40. The solid line is our T2M model, dashed lines are for the baseline models.} 
    \label{fig:top_k}
\end{figure}

\begin{table}[t!]
    \centering
    \setlength{\tabcolsep}{4pt}
    \scalebox{0.8}{
    \begin{tabular}{ l | c c c} 
        & \multicolumn{3}{c}{AwA Super Sets}
        \cr & Seen   & Unseen & Harmonic\\
        \midrule
        DeViSE &$53.0 \pm 1.9$& $50 \pm 0.6$& $51.5 \pm 0.9$ \\
        DEM  & $50.1 \pm 1.4$ & $48.3 \pm 1.8$ &  $49.2 \pm 1.6$\\
        CIZSL   &$57.3 \pm 5.6$& $50.2 \pm 5.8$&$55.0 \pm 4.0$ \\
        GRaWD &$59.8 \pm 3.5$& $51.6 \pm 4.8$&$55.3 \pm 3.1$ \\
        T2M-HN (ours) & $\mathbf{67.2 \pm 5.2}$ & $\mathbf{57.3 \pm 5.7}$ & $\mathbf{61.9 \pm 5.4}$ \\
        \bottomrule
    \end{tabular}
    }
    \caption{\textbf{Classification using attributes}. Values denote the Area under the Recall-Precision curve averaged over the 13 test attributes $\pm$ s.e.m. over these attributes. The seen results are new images from the seen classes, while the unseen results are images from unseen classes. Both are evaluated when classifying only the test attributes. The full protocol is in \ref{app:supersets}. 
    }
    \label{tab:set-of-classes}
\end{table}

\subsection{The Impact of Equivariance Design on HNs}
\label{app:ablation}

To evaluate the effect of the equivariance property  on our HN-based model performance, we compared variants with and without the equivariance design. We repeat the experiment for an on-demand model with one or two fully connected layers. 
Figure \ref{fig:results-ablation} shows the mean accuracy of the following variants: 
\textbf{(1) T2M-HN 1-layer} An equivariant HN that predicts one equivariant FC layer;
\textbf{(2) 1-layer w.o. EV} A FC HN that predicts one fully connected layer;
\textbf{(3) T2M-HN 2-layers } An equivariant HN that predicts two FC layers for the on-demand model: The first is invariant and the second is equivariant; and
\textbf{(4) 2-layer w.o. EV} A FC HN that predicts two FC layers.
There is no big difference in the number of parameters between EV and non-EV architectures:
(1) T2M-HN 1-layer - 244K parameters,
(2) 1-layer w.o. EV - 183K parameters,
(3) T2M-HN 2-layers - 7.6M parameters,
(4) 2-layer w.o. EV - 7.5 M parameters.

In all cases, the equivariant HN performs better than the simple fully connected.
For AwA, T2M-HN 1-layer performs better than T2M-HN 2-layers. We believe this is because ResNet backbone separates the images to be linearly separable. For BABEL, we used 2s-AGCN as a features extractor and in that case, T2M-HN 2-layer generalizes better to unseen classes.

\begin{figure}[h!]
    \centering
    \includegraphics[width=0.49\textwidth]{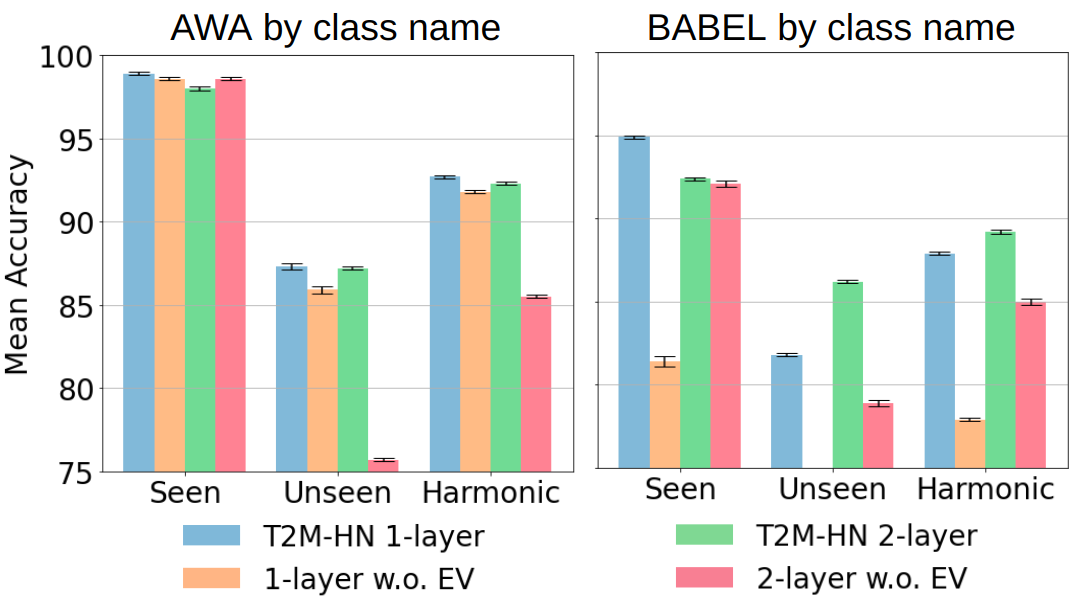} 
    \caption{Ablation study. Mean classification accuracy (averaged across class pairs) on seen and unseen classes and their harmonic mean for the AWA and BABEL datasets. }
    \label{fig:results-ablation}
\end{figure}

\section{AwA GPT-3 descriptions}
\label{app:gpt-examples}
We use GPT3 \cite{gpt3brown2020} to generate 5 synthetic descriptions for each class of AwA. During training and evaluation, we randomly choose one description for each class in the batch, from its corresponding 5 class descriptions.
We use the API provided by OpenAI to ask "text-davinci-002" engine with a temperature of 0, max tokens of 512, and the prompt: "Suggest 5 definitions for an animal. Animal: \{animal\_name\}. Definitions:"

\noindent\textbf{Animal: moose}\newline
Definitions:
\begin{enumerate}[noitemsep]
    \item A large, dark-colored deer with enormous antlers, native to North America and Europe.
    \item An animal of the deer family with humped shoulders, long legs, and a large head with antlers.
    \item A large, awkward-looking mammal with a long face and humped shoulders.
    \item A very large deer with antlers that can spread six feet or more from tip to tip.
    \item The largest member of the deer family, with males weighing up to 1,800 pounds and having antlers that can spread up to six feet from tip to tip.
\end{enumerate}

\noindent\textbf{Animal: spider monkey}\newline
Definitions:
\begin{enumerate}[noitemsep]
    \item A type of monkey that has long legs and arms and a long tail.
    \item A monkey that is found in the rainforests of Central and South America.
    \item A monkey that is known for its acrobatic abilities.
    \item A monkey that is considered to be one of the most intelligent primates.
    \item A monkey that is endangered in many parts of its range.
\end{enumerate}

\noindent\textbf{Animal: rhinoceros} \newline
Definitions:
\begin{enumerate}[noitemsep]
    \item A large, thick-skinned mammal with one or two horns on its snout, native to Africa and southern Asia.
    \item An animal that is hunted for its horn, which is used in traditional Chinese medicine.
    \item A large, herbivorous mammal with a single horn on its nose, found in Africa and southern Asia.
    \item A mammal of the family Rhinocerotidae, having thick, grey or brown skin and one or two horns on the snout.
    \item A very large, plant-eating mammal with one or two horns on its nose, found in Africa and southern Asia.
\end{enumerate}

\textbf{Elephant}:
\begin{enumerate}[noitemsep]
    \item The largest land animal in the world, with males weighing up to six tons.
    \item A plant-eating mammal with a long trunk, large ears, and thick, grey skin.
    \item A mammal of the family Elephantidae, having a long trunk, large ears, and thick, grey skin.
    \item An intelligent animal that is known for its memory and its ability to use its trunk for a variety of tasks.
    \item An endangered species that is hunted for its ivory tusks.
\end{enumerate}

\section{Data splits}
\label{app:dataset-splits}

\textbf{SUN unseen classes:} 'volcano', 'poolroom establishment',  'veterinarians office',  'reception', 'field wild', 'diner indoor', 'garbage dump', 'server room', 'vineyard', 'jewelry shop', 'drugstore', 'herb garden', 'lock chamber', 'temple east asia', 'marsh', 'cottage garden', 'cathedral outdoor', 'dentists office', 'pharmacy', 'hangar indoor', 'volleyball court indoor', 'lift bridge', 'synagogue outdoor', 'boathouse', 'ice shelf', 'boxing ring', 'rope bridge', 'electrical substation', 'auditorium', 'chalet', 'booth indoor', 'wine cellar barrel storage',
 'greenhouse outdoor', 'badminton court indoor', 'thriftshop',
 'cemetery',
 'rainforest',
 'courtyard',
 'underwater coral reef',
 'formal garden',
 'ice skating rink outdoor',
 'palace',
 'movie theater indoor',
 'dinette home',
 'sandbar',
 'ball pit',
 'amphitheater'

\textbf{SUN seen classes:}
All remaining classes.

\textbf{ModelNet40:} We follow \cite{cheraghian2022zero-PCZS, DBLP:conf/mva/CheraghianRP19, DBLP:conf/3dim/MicheleBPBM21} and use the 10 classes included in ModelNet-10 as unseen classes, and the other 30 as seen. 

\textbf{BABEL unseen classes:} 'a pose', 'action with ball', 'adjust', 'catch', 'clean something', 'communicate (vocalise)', 'crawl', 'get injured', 'hand movements', 'hop', 'limp', 'mix', 'play sport', 'press something', 'rolling movement', 'shuffle', 'side to side movement', 'sneak', 'spread', 'support', 'swing body part', 'trip', 'upper body movements', 'wait'

\textbf{BABEL seen classes: } All remaining classes.

\textbf{CUB unseen classes: } 'Acadian Flycatcher', 'American Crow', 'American Three Toed Woodpecker', 'Baltimore Oriole', 'Bank Swallow', 'Belted Kingfisher', 'Black Billed Cuckoo', 'Black Footed Albatross', 'Black Throated Sparrow', 'Boat Tailed Grackle', 'Bohemian Waxwing', 'Brandt Cormorant', 'Brewer Blackbird', 'Cape May Warbler', 'Cedar Waxwing', 'Chestnut Sided Warbler', 'Field Sparrow', 'Golden Winged Warbler', 'Grasshopper Sparrow', 'Gray Crowned Rosy Finch', 'Great Crested Flycatcher', 'Great Grey Shrike', 'Groove Billed Ani', 'Hooded Oriole', 'Horned Grebe', 'Indigo Bunting', 'Least Auklet', 'Least Tern', 'Marsh Wren', 'Mockingbird', 'Northern Flicker', 'Northern Waterthrush', 'Pacific Loon', 'Pied Billed Grebe', 'Pomarine Jaeger', 'Purple Finch', 'Red Legged Kittiwake', 'Rhinoceros Auklet', 'Sayornis', 'Scott Oriole', 'Tree Sparrow', 'Tree Swallow', 'Western Grebe', 'Western Gull', 'Western Wood Pewee', 'White Breasted Kingfisher', 'White Eyed Vireo', 'White Pelican', 'Wilson Warbler', 'Yellow Bellied Flycatcher', 'Yellow Billed Cuckoo'

\textbf{CUB seen classes: } All remaining classes.

\section{Attributes used for one-class classification}
\label{app:supersets}
As mentioned in section \ref{sec:set-of-classes}, we use some of the attributes from the AwA dataset to define one-class classification tasks. First, we removed non-visual attributes. Then, we randomly split the remaining 53 attributes into 30 train, 10 validation, and 13 test attributes. We split both the images and the attributes, constructing 4 groups of images and attributes: (1) \textit{Training images} from training attributes and training classes, used to train the hypernetwork; (2) \textit{Validation images} from the training classes, with the validation attributes used to tune hyperparameters; (3) \textit{Test images from seen classes}, new images of test attributes, whose class was seen during training (but not the specific images); and (4) \textit{Test images from unseen classes}, new images of test attributes, whose class was not seen during training. We report the average Area under the Recall-Precision curve over seen (group (3)) and unseen classes (group (4)).
The results are shown in Figure \ref{fig:one_class} and in Table \ref{tab:set-of-classes}. The attributes split is as follows:

\textbf{AwA train attributes:} 'orange', 'red', 'longneck', 'horns', 'tusks', 'flys', 'desert', 'cave', 'jungle', 'water', 'bush', 'lean', 'forest', 'gray', 'strainteeth', 'stripes', 'mountains', 'arctic', 'paws', 'hooves', 'pads', 'small', 'furry', 'ground', 'patches', 'white', 'fields', 'bipedal', 'toughskin', 'plains'.

\textbf{AwA validation attributes:} 'buckteeth', 'chewteeth', 'yellow', 'hairless', 'bulbous', 'big', 'flippers', 'tree', 'walks', 'coastal'.

\textbf{AwA test attributes:} 'quadrapedal', 'black', 'blue', 'ocean', 'longleg', 'spots', 'hands', 'claws', 'muscle', 'meatteeth', 'tail', 'brown', 'swims'.